\documentclass[a4paper]{article}

\usepackage{graphicx}%
\usepackage{multirow}%
\usepackage{amsmath,amssymb,amsfonts}%
\usepackage{amsthm}%
\usepackage{mathrsfs}%
\usepackage[title]{appendix}%
\usepackage{xcolor}%
\usepackage{textcomp}%
\usepackage{manyfoot}%
\usepackage{booktabs}%
\usepackage{algorithm}%
\usepackage{algorithmicx}%
\usepackage{algpseudocode}%
\usepackage{listings}%

\usepackage{subfigure}
\usepackage{times}
\usepackage{algpseudocode}
\usepackage{bm}
\usepackage{xr}
\usepackage{url}
\usepackage{natbib}

\newcommand{\R}[1]{\mathbb{R}^{#1}}
\newcommand{\Jdv}{J_\mathrm{DV}}
\newcommand{\hJdv}{\widehat{J}_\mathrm{DV}}
\newcommand{\bs}{\bm{s}}
\newcommand{\tbs}{\tilde{\bm{s}}}
\newcommand{\ts}{\tilde{s}}

\newcommand{\bq}{\bar{q}}
\newcommand{\bbq}{\bar{\bm{q}}}

\newcommand{\bu}{\bm{u}}
\newcommand{\bv}{\bm{v}}

\newcommand{\bx}{\bm{x}}
\newcommand{\ps}{p_{\mathrm{s}}}
\newcommand{\px}{p_{\mathrm{x}}}
\newcommand{\pss}{p_{\mathrm{ss'}}}

\newcommand{\psv}{p_{\mathrm{sv}}}

\newcommand{\pxcu}{p_{\mathrm{x|u}}}

\newcommand{\pwx}{p_{\mathrm{w}|\mathrm{xx'}}}
\newcommand{\pwxu}{p_{\mathrm{w}|\mathrm{xu}}}
\newcommand{\pws}{p_{\mathrm{w}|\mathrm{ss'}}}

\newcommand{\pwsv}{p_{\mathrm{w}|\mathrm{sv}}}
\newcommand{\psvw}{p_{\mathrm{sv}|\mathrm{w}}}

\newcommand{\pw}{p_{\mathrm{w}}}

\newcommand{\pxx}{p_{\mathrm{xx'}}}
\newcommand{\pxcx}{p_{\mathrm{x|x'}}}
\newcommand{\pwxx}{p_{\mathrm{wx\tilde{x}}}}
\newcommand{\bh}{\bm{h}}

\newcommand{\bg}{\bm{g}}
\newcommand{\tbg}{\tilde{\bm{g}}}
\newcommand{\tg}{\tilde{g}}

\renewcommand{\tbs}{\tilde{\bm{s}}}

\renewcommand{\ts}{\tilde{s}}
\newcommand{\tbv}{\tilde{\bm{v}}}
\newcommand{\tv}{\tilde{v}}

\newcommand{\bphi}{\bm{\phi}}
\newcommand{\bvphi}{\bm{\varphi}}
\newcommand{\tbvphi}{\tilde{\bm{\varphi}}}
\newcommand{\tbphi}{\tilde{\bm{\phi}}}

\newcommand{\diag}{\mathrm{diag}}

\newcommand{\bD}{\bm{D}}

\newcommand{\calI}{\mathcal{I}}

\newcommand{\bw}{\bm{w}}
\newcommand{\calV}{\mathcal{W}}
\newcommand{\dx}{d_{\mathrm{x}}}
\newcommand{\ds}{d_{\mathrm{s}}}
\newcommand{\dw}{d_{\mathrm{w}}}

\newcommand{\du}{d_{\mathrm{u}}}
\newcommand{\dv}{d_{\mathrm{v}}}
\newcommand{\intd}{\mathrm{d}}
\newcommand{\calM}{\mathcal{M}}
\newcommand{\pxuw}{p_{\mathrm{xu}|\mathrm{w}}}
\newcommand{\pxu}{p_{\mathrm{xu}}}

\newcommand{\bG}{\bm{G}}
\newcommand{\bJ}{\bm{J}}

\newcommand{\qw}{q_{\mathrm{w}}}
\newcommand{\bqw}{\bar{q}_{\mathrm{w}}}
\newcommand{\qsv}{q_{\mathrm{sv}}}

\newcommand{\parder}[2]{\frac{\partial #1}{\partial #2}}

\def\<{\langle}
\def\>{\rangle}
\theoremstyle{thmstyleone}
\newtheorem{Theo}{Theorem}
\newtheorem{Prop}[Theo]{Proposition} 
\newtheorem{Lemm}[Theo]{Lemma} 
\newtheorem{Coro}[Theo]{Corollary}%

\theoremstyle{thmstyletwo}

\theoremstyle{thmstylethree}

\title{Identifiability of a statistical model with two latent vectors:
Importance of the dimensionality relation and application to graph
embedding}
\author{Hiroaki Sasaki\\
Department of Mathematical Informatics\\
Meiji Gakuin University, Japan}

\date{}

\begin{document}
\maketitle
\abstract{Identifiability of statistical models is a key notion in
 unsupervised representation learning. Recent work of nonlinear
 independent component analysis (ICA) employs auxiliary data and has
 established identifiable conditions. This paper proposes a statistical
 model of two latent vectors with single auxiliary data generalizing
 nonlinear ICA, and establishes various identifiability
 conditions. Unlike previous work, the two latent vectors in the
 proposed model can have arbitrary dimensions, and this property enables
 us to reveal an insightful dimensionality relation among two latent
 vectors and auxiliary data in identifiability conditions. Furthermore,
 surprisingly, we prove that the indeterminacies of the proposed model
 has the same as \emph{linear} ICA under certain conditions: The
 elements in the latent vector can be recovered up to their permutation
 and scales. Next, we apply the identifiability theory to a statistical
 model for graph data. As a result, one of the identifiability
 conditions includes an appealing implication: Identifiability of the
 statistical model could depend on the maximum value of link weights in
 graph data. Then, we propose a practical method for identifiable graph
 embedding. Finally, we numerically demonstrate that the proposed method
 well-recovers the latent vectors and model identifiability clearly
 depends on the maximum value of link weights, which supports the
 implication of our theoretical results.}

 \section{Introduction}
 \label{sec:intro}
 Unsupervised representation learning is a big issue in machine
 learning~\citep{raina2007self,bengio2013representation}: Data
 representation learned from a large amount of unlabeled data has been
 applied for downstream tasks to improve the performance. Empirical
 successes of unsupervised representation learning have been reported in
 a number of works such as classification~\citep{klindt2020towards},
 natural language processing~\citep{peters2018deep,devlin2019bert},
 graph embedding~\cite{cai2018comprehensive}. On the other hand, it is
 often theoretically challenging to understand how representation of
 data is learned.
 
 One of the key notions in unsupervised representation learning is
 identifiability of statistical models. A seminar work is \emph{linear}
 independent component analysis (ICA)~\citep{comon1994independent} where
 data vectors are generated from a linear mixing of latent variables.
 It has been proved than the latent variables can be identified from
 observed data vectors up to their permutation and scales. The
 fundamental conditions in linear ICA are statistical independence and
 nonGaussianity of the latent variables. The natural extension of linear
 ICA is \emph{nonlinear} ICA where a general nonlinear function of
 latent variables is assumed in data generation. However, it has been
 shown that nonlinear ICA is considerably difficult under the same
 independence condition as linear ICA because there exist an infinite
 number of independent representations of
 data~\citep{hyvarinen1999nonlinear,pmlr-v97-locatello19a}.

 More recently, identifiable statistical models of single latent vectors
 have been proved in nonlinear
 ICA~\citep{hyvarinen2016unsupervised,hyvarinen2018nonlinear,sasaki2022representation}.
 In contrast with linear ICA, these works assume to have additional data
 called the auxiliary data, and make a conditional independence
 condition of latent variables given auxiliary data, which is
 alternative to the independent condition in linear ICA.  Based on these
 conditions, it has been proved that latent variables are identifiable
 up to their permutation and elementwise nonlinear functions. Since
 then, a variety of nonlinear ICA methods haven been proposed such as
 time series data~\citep{pmlr-v54-hyvarinen17a} and variational
 autoencoder~\citep{pmlr-v108-khemakhem20a}, and applied to causal
 analysis~\citep{monti2019causal,pmlr-v108-wu20b} and transfer
 learning~\citep{pmlr-v119-teshima20a}. However, in nonlinear ICA, the
 indeterminacy of the elementwise nonlinear functions might make
 interpretation of the learned representation of data difficult even for
 simulated artificial data.

 This paper proposes a new statistical model of two latent vectors with
 single auxiliary data and establishes identifiability conditions for
 one or both of the two latent vectors. Previously, an energy-based
 model (EBM)~\citep{sasaki2018neural,khemakhem2020ice} related to two
 latent vectors has been proposed. EBM makes use of statistical
 dependency of the two latent vectors in identifiability conditions, and
 assumes that these two vectors have the same dimension. In contrast,
 the proposed model allows these latent vectors to have different
 dimensions. Furthermore, thanks to additional auxiliary data, unlike
 EBM, these latent vectors are not necessarily dependent and can be even
 independent. Thus, our statistical model is very flexible. This
 flexibility reveals interesting dimensionality conditions between these
 latent vectors and auxiliary data: For instance, if one wants to
 recover one latent vector, then the dimensionality of the other latent
 vector should be equal to or larger than that of the target latent
 vector.  In addition, we show that the indeterminacies of our model are
 the same as \emph{linear} ICA under certain conditions. Therefore,
 surprisingly, the indeterminacies are simply permutation and scales of
 the elements in the latent vector, while data vectors are generated as
 a nonlinear mixing of latent variables. This is a significant step to
 interpret data representation in the identifiable models.
 
 A promising application of the proposed statistical model is graph
 data, which consists of data vectors associated with nodes and their
 discrete link weights. By assuming that data vectors are generated from
 latent vectors, applying our identifiability theory to graph data
 includes an appealing implication: In order to recover the latent
 vectors underlying data, the maximum value of link weights should be
 large relative to the dimensionality of the latent vector.
 Based on the identifiability conditions, we propose a practical method
 for identifiable graph embedding called \emph{graph component analysis}
 (GCA). Numerical experiments through GCA supports our theoretical
 implication and demonstrates that the performance of GCA clearly
 depends on the the maximum value of link weights.
 

	


 \section{Data generative model of two latent vectors}
 \label{sec:generative-model}
 This section proposes a statistical model of two latent vectors with
 single auxiliary data. Given either discrete or continuous auxiliary
 data $\bw\in\calV\subset\R{\dw}$ with the marginal distribution
 $\pw(\bw)$, we assume that two vectors of continuous latent variables,
 $\bs=(s^{(1)},\dots,s^{(\ds)})^{\top}\in\R{\ds}$ and
 $\bv=(v^{(1)},\dots,v^{(\dv)})^{\top}\in\R{\dv}$, follow the
 conditional distribution $\psvw(\bs,\bv|\bw)$. Then, two data vectors
 $\bx\in\R{\dx}$ and $\bu\in\R{\du}$ are generated from
 \begin{align}
  \bx=\bm{f}(\bs)~~\text{and}~~\bu=\bvphi(\bv),
  \label{generative-model}
 \end{align}
 where $\bm{f}:\R{\ds}\to{\R{\dx}}$ and $\bvphi:\R{\dv}\to{\R{\du}}$ are
 nonlinear mixing functions with $\ds\leq{\dx}$ and $\dv\leq{\du}$. We
 specifically assume that $\bm{f}$ and $\bvphi$ are smooth
 embeddings\footnote{A smooth embedding is a smooth immersion that is
 homeomorphism on its image~\citep{lee2012introduction}.}. Thus,
 $\bm{f}$ and $\bvphi$ are continuous and bijective functions with
 smooth inverse to their images, which are denoted by
 \begin{align*}
  \calM_{\bm{f}}&:=\{\bx=\bm{f}(\bs)~|~\bs\in\R{\ds}\}\subset\R{\dx}\\
  \calM_{\bvphi}&:=\{\bu=\bvphi(\bv)~|~\bv\in\R{\dv}\}\subset\R{\du}.
 \end{align*}
 Proposition~5.2 in~\citet{lee2012introduction} allows us to regard
 $\calM_{\bm{f}}$ and $\calM_{\bvphi}$ as embedded submanifolds of
 $\R{\dx}$ and $\R{\du}$, respectively.

 A motivating example of this latent variable model is graph
 embedding~\citep{cai2018comprehensive}. The goal of graph embedding is
 to learn an embedding function from data vectors associated with nodes
 and their link weights (i.e., graph data), which could be interpreted
 as an estimate of the inverse of $\bm{f}$ or $\bvphi$
 in~\eqref{generative-model}. Thus, graph embedding can be seen as
 unsupervised representation learning for graph data.  By regarding two
 data vectors at a pair of nodes (e.g., $\bx$ and $\bx'$) and link
 weight (e.g., $w$) as $\bx,\bu$ and $\bw$ in our model respectively,
 the proposed model includes a latent variable model for graph data as a
 special case. By applying our results, we establish identifiability
 conditions for graph data one of which includes an appealing
 implication.

 Let us clarify that the observable variables are $\bw$, $\bx$ and
 $\bu$, while $\bs$ and $\bv$ are latent variables and cannot be
 directly observed. Next, after defining the notion of identifiability
 for this statistical model, we establish the identifiability
 conditions.
 \section{Identifiability theory}
 \label{sec:identifiability}
 We first define the notion of identifiability in this paper, and then
 establish various identifiability conditions.
  \subsection{Definition of identifiability}
  Here, the definition of identifiability is given. Since $\bm{f}$ and
  $\bvphi$ are smooth embeddings (i.e., continuous and bijective
  functions) to $\calM_{\bm{f}}$ and $\calM_{\bphi}$, we express $\bs$
  and $\bv$ as
  \begin{align}
   \bs=\bg(\bx)~~\text{and}~~\bv=\bphi(\bu), \label{inv-model}   
  \end{align}  
  where $\bg:\calM_{\bm{f}}\to\R{\ds}$ and
  $\bphi:\calM_{\bphi}\to\R{\ds}$ are the inverses of $\bm{f}$ and
  $\bvphi$, respectively. By supposing the Euclidean metric both in the
  latent spaces of $\bs$ and $\bv$,~\eqref{inv-model} enables us to
  express the conditional distribution of $\bx\in\calM_{\bm{f}}$ and
  $\bu\in\calM_{\bvphi}$ given $\bw$~\citep{gemici2016normalizing} as
  \begin{align}
   &\pxuw^{\bg,\bphi}(\bx,\bu|\bw)
   =\psvw(\bg(\bx),\bphi(\bu)|\bw)
   \sqrt{\det(\bG_{\bg}(\bx))}
   \sqrt{\det(\bG_{\bphi}(\bu))},
   \label{cond-xuw}
  \end{align}
  where $\bG_{\bg}(\bx):=\bJ_{\bg}(\bx)^{\top}\bJ_{\bg}(\bx)$ and
  $\bG_{\bphi}(\bu):=\bJ_{\bphi}(\bu)^{\top}\bJ_{\bphi}(\bu)$.  Then,
  the notion of \emph{partially identifiability} is defined as follows:
  $\pxuw^{\bg,\bphi}$ is said to be partially identifiable with
  respective to $\bg$ when
  \begin{align}
   \pxuw^{\bg,\bphi}(\bx,\bu|\bw)&=\pxuw^{\tbg,\tbphi}(\bx,\bu|\bw)
   ~~\Rightarrow~~\bg=\tbg.
   \label{patial-identifiability}
  \end{align}
  The partial identifiability might be sufficient in many practical
  situations, yet we also define an even stronger notion called the
  \emph{full identifiability}: If
  \begin{align}
   \pxuw^{\bg,\bphi}(\bx,\bu|\bw)&=
   \pxuw^{\tbg,\tbphi}(\bx,\bu|\bw)
   ~\Rightarrow~\bg=\tbg~~\text{and}~~\bphi=\tbphi,
   \label{full-identifiability}
  \end{align}
  $\pxuw^{\bg,\bphi}$ is said to be fully identifiable. It follows from
  the definition that the full identifiability guarantees that one can
  estimate $\bg$ and $\bphi$ through estimation of the conditional
  distribution of $\bx$ and $\bu$ given $\bw$. Let us note that the
  recovery of two (single) latent vectors is almost a synonym of full
  (partial) identifiability by~\eqref{inv-model} in this paper. Next, we
  establish various conditions for partial and full identifiability.
  \subsection{Partial identifiability}
  We first show some conditions for partial identifiability with respect
  to $\bg$. Note that the results can be easily converted to partial
  identifiability with respect to $\bphi$. To this end, we specifically
  assume that the conditional distribution $\psvw(\bs,\bv|\bw)$ is given
  by
  \begin{align}
   \log\psvw(\bs,\bv|\bw)
   =\sum_{i=1}^{\ds}q^{(i)}(s^{(i)},\bv,\bw)-\log{Z}(\bw),
   \label{cond-svw1}
  \end{align}
  where $Z(\bw)$ denotes the partition function and $q^{(i)}(s,\bv,\bw)$
  for $i=1,\dots,\ds$ are differentiable functions with respect to $s$,
  $\bv$ and $\bw$. For better exposition of theorem, we define a couple
  of notations: With a fixed point $\bar{\bw}$ of $\bw$,
  $\bq^{(i)}(s,\bv,\bw):=q^{(i)}(s,\bv,\bw)-q^{(i)}(s,\bv,\bar{\bw})$.
  Then, we denote the first- and second-order partial derivatives of
  $\bq^{(i)}(s,\bv,\bw)$ by
  $\bq^{(i,1)}(s,\bv,\bw):=\frac{\partial}{\partial{s}}\bq^{(i)}(s,\bv,\bw)$
  and
  $\bq^{(i,2)}(s,\bv,\bw):=\frac{\partial^2}{\partial{s}^2}\bq^{(i)}(s,\bv,\bw)$,
  respectively and compactly express these derivatives as the following
  single $2\ds$-dimensional vector:
  \begin{align*}
   \bbq^{\prime}(\bs,\bv,\bw)&:=
   \left(\bq^{(1,1)}(s^{(1)},\bv,\bw),\dots,
   \bq^{(\ds,1)}(s^{(\ds)},\bv,\bw),\right.\\
   &\left.\qquad\qquad\qquad
   \bq^{(1,2)}(s^{(1)},\bv,\bw),
   \dots,\bq^{(\ds,2)}(s^{(\ds)},\bv,\bw) \right)^{\top}\in\R{2\ds}.
  \end{align*}
  We are now ready to present the following theorem for partial
  identifiability, which implies that the dimensionality relation among
  latent variables and auxiliary data can be important:
  \begin{Theo}
   \label{theo:partial-IDF} Suppose that the following assumptions hold:
   \begin{enumerate}
    \item[(A1)] The conditional distribution $\psvw(\bs,\bv|\bw)$ is
		expressed as~\eqref{cond-svw1}.

    \item[(A2)] Dimensionality assumption: $2\ds\leq{\dv+\dw}$.

    \item[(A3)] There exist two fixed points $\bar{\bv}$ and $\bv'$ of
		$\bv$ and $\bar{\bw}$ and $\bw'$ of $\bw$ such that the
		rank of the following $2\ds$ by $(\dv+\dw)$ matrix is of
		full-rank for all $\bs$:
		\begin{align*}
		 \hspace{-8mm}
		 [\bbq'(\bs,\bv^{(1)},\bw),\dots,\bbq'(\bs,\bv^{(\dv)},\bw),\bbq'(\bs,\bv,\bw^{(1)}),\dots,\bbq'(\bs,\bv,\bw^{(\dw)})]
		 \bigr|_{(\bv,\bw)=(\bv',\bw')},
		\end{align*}
		where the $i$-th elements in $\bv$ and $\bw$ are fixed
		at the $i$-th ones in $\bar{\bv}$ and $\bar{\bw}$ as
		\begin{align*}
		 \bv^{(i)}&=(v^{(1)},\dots,v^{(i)}=\bar{v}^{(i)}\dots,v^{(\dv)})^\top\\
		 \bw^{(i)}&=(w^{(1)},\dots,w^{(i)}=\bar{w}^{(i)}\dots,w^{(\dw)})^\top.
		\end{align*}
   \end{enumerate}
   Then, $\pxuw^{\bg,\bphi}$ is partially identifiable with respective
   to $\bg$ up to a permutation of the elements in $\bg$ and elementwise
   nonlinear functions.
  \end{Theo}  
  The proof is given in Appendix~\ref{app:proof-partial-IDF}.  The
  conclusion is essentially the same as existing works of nonlinear ICA,
  but Assumption~(A2) seems appealing: First of all, it directly depends
  on the dimensionalities of latent vectors $\bs$ and $\bv$ (i.e., $\ds$
  and $\dv$), \emph{not} on those of the observable data vectors $\bx$
  and $\bu$ (i.e., $\dx$ and $\du$). Second, the sum of the
  dimensionalities of $\bv$ and $\bw$ should be twice as large as or
  larger than the dimensionality of $\bs$. This implies that partial
  identifiability could be difficult when both $\dv$ and $\dw$ are very
  low-dimensional vectors. Previously, a similar dimensionality
  condition has been established in~\citet{sasaki2022representation},
  but unlike our work, it directly depends on $\dx$ and $\du$. Thus, our
  condition can be very weak. Assumption~(A1) is akin to the well-known
  conditional independence assumption in nonlinear ICA. Assumption~(A3)
  implies that $\bv$ and $\bw$ are statistically dependent to $\bs$.
  For instance, if $\bw$ is independent to $\bs$,
  $\bq^{(i,1)}(s,\bv,\bw)$ and $\bq^{(i,2)}(s,\bv,\bw)$ are functions of
  $s$ and $\bv$. Thus, the matrix in Assumption~(A3) cannot be of
  full-rank when $2\ds>\dv$.  Our proof is similar yet includes two
  extentions over~\citet[Proposition~1]{sasaki2022representation}:
  First, the latent vectors lie on embedded lower-dimensional manifolds
  of the respective input data spaces in this paper,
  while~\citet{sasaki2022representation} assumes that the latent vector
  $\bs$ has the same dimensionality as the input data vector $\bx$
  (i.e., $\dx=\ds)$. Second, auxiliary variables $\bw$ can be either
  discrete or continuous. On the other hand, it has been restricted to
  be continuous
  in~\citet{sasaki2022representation}. \citet{hyvarinen2018nonlinear}
  also supposes that auxiliary variables can be either discrete or
  continuous, but in contrast with us, a useful dimensionality condition
  has not been derived.
  
  Our generative model~\eqref{generative-model} can be interpreted as a
  restriction to the following \emph{single} generative model:
  \begin{align}
   \left(
   \begin{array}{l}
    \bx\\\bu
   \end{array}
   \right)
   =\bm{F}(\bs,\bv),
   \label{generative-model-single}
  \end{align}
  where $\bm{F}:\R{\ds+\dv}\to\R{\dx+\du}$ is a mixing function. Then,
  the (full) identifiability with respect to $\bm{F}^{-1}$ is
  essentially the same problem as nonlinear ICA with single auxiliary
  data $\bw$. In fact, by applying almost the same
  proof\footnote{Alternatively, we need to assume that
  $\psvw(\bs,\bv|\bw)$ takes the form of
  $\log\psvw(\bs,\bv|\bw)=\sum_{i=1}^{\ds}q_{\mathrm{s}}^{(i)}(s^{(i)},\bw)+\sum_{j=1}^{\dv}q_{\mathrm{v}}^{(j)}(v^{(j)},\bw)
  -\log{Z}(\bw)$} as Theorem~\ref{theo:partial-IDF}, we obtain the same
  conclusion yet the dimensionality assumption has to be modified to
  $2(\ds+\dv)\leq{\dw}$. In comparing with Assumption~(A2), this
  assumption is stronger and one may need to prepare for relatively
  high-dimensional auxiliary data $\bw$. Thus, our work can be seen as
  making a weaker assumption by restricting the single generative model.
  
  A similar energy-based model (EBM) has been proposed
  in~\citet{sasaki2018neural} and~\citet{khemakhem2020ice} as
  \begin{align}
   \pxcu^{\bg,\bphi}(\bx|\bu)
   =\frac{1}{Z(\bu;\bg,\bphi)}\exp(\bg(\bx)^{\top}\bphi(\bu)),
   \label{cond-xu}
  \end{align}
  where $\pxcu^{\bg,\bphi}(\bx|\bu)$ denotes the conditional
  distribution of $\bx$ given $\bu$. Based on
  $\pxcu^{\bg,\bphi}(\bx|\bu)$, identifiability conditions for $\bg$
  and/or $\bphi$ have been established. Compared with
  us,~\eqref{cond-xu} is restrictive because the dimensionality of $\bg$
  has to be the same as $\bphi$. In contrast, we do have no
  dimensionality restrictions to $\bg$ and $\bphi$, which consequently
  enables us to derive an interesting Assumption~(A2). In
  addition,~\eqref{cond-xu} could be regarded as a special case of our
  conditional distribution~\eqref{cond-svw1} without $\bw$ where
  $q^{(i)}(g^{(i)}(\bx),\bphi(\bu),\bw)=g^{(i)}(\bx)\phi^{(i)}(\bu)$.
  
  The dimensionality Assumption~(A2) might be strong in practice when
  $\dw$ is very small as in graph data where $\bw$ is often
  one-dimensional data (i.e., $\dw=1$). The next theorem establishes a
  more relaxed dimensionality condition by introducing a novel
  restriction to the conditional distribution $\psvw(\bs,\bv|\bw)$:
  \begin{Theo}
   \label{theo:partial-IDF-mild} Suppose that the following assumptions
  hold:
   \begin{enumerate}
    \item[(B1)] The conditional distribution $\psvw(\bs,\bv|\bw)$ is
		given by~\eqref{cond-svw1} with
		$q^{(i)}(s,\bv,\bw)=\qsv^{(i)}(s,\bv)\qw^{(i)}(\bw)$.
		
    \item[(B2)] Dimensionality assumption: $\ds\leq{\dv}$.

    \item[(B3)] There exists a single fixed point $\bv'$ of $\bv$ such
		that a $\ds$ by $\dv$ matrix with the $(l,m)$-th element
		$\frac{\partial^2}{\partial{s^{(l)}}\partial{v^{(m)}}}\qsv^{(l)}(s^{(l)},\bv)$
		is of full rank at $\bv=\bv'$ and for all $\bs$.

    \item[(B4)] There exist $\ds+1$ fixed points, $\bar{\bw},
		\bw(1),\dots,\bw(\ds)$, of $\bw$ such that the
		followings holds:
		\begin{itemize}
		 \item
		      $\bqw^{(i)}(\bw):=\qw^{(i)}(\bw)-\qw^{(i)}(\bar{\bw})\neq0$
		      for all $i=1,\dots,\ds$ and
		      $\bw\in\{\bw(1),\dots,\bw(\ds)\}$.
		      
		 \item $\ds$-dimensional vectors
		      $\bm{\bq}_{\mathrm{w}}(\bw(1)),\dots,\bm{\bq}_{\mathrm{w}}(\bw(\ds))$
		      are linearly independent defined by
		       \begin{align*}
			\bm{\bq}_{\mathrm{w}}(\bw)
			:=\left(\bqw^{(1)}(\bw),\dots,
			\bqw^{(\ds)}(\bw)\right)^{\top}.
		       \end{align*}
		\end{itemize}		
   \end{enumerate}
   Then, $\pxuw^{\bg,\bphi}$ is partially identifiable with respective
   to $\bg$ up to a permutation of the elements in $\bg$ and elementwise
   nonlinear functions.
  \end{Theo}
  The proof is given by Appendix~\ref{app:partial-IDF-mild}.
  Assumption~(B2) requires that the dimensionality of latent variables
  $\bv$ is more than or equal to $\ds$, and is useful when
  low-dimensional auxiliary data $\bw$ is only available and
  Assumption~(A2) in Theorem~\ref{theo:partial-IDF} cannot be satisfied.
  On the other hand, this usefulness costs at simplifying the
  conditional distribution $\psvw$ as in Assumption~(B1). A similar
  assumption as Assumption (B1) can be found
  in~\citet[Eq.(13)]{sasaki2022representation}, but our assumption is
  more general because it can be seen as a special case of
  Assumption~(B1) where $\qsv^{(i)}(s,\bv)=\qsv^{(i)}(s,v^{(i)})$ and
  $\qw^{(i)}(\bw)$ is a constant function. Assumption~(B4) implies that
  $\psvw$ is sufficiently diverse. For instance, if $\qw(\bw)$ is a
  constant, Assumption~(B4) is never satisfied.
  \subsection{Full identifiability}
  Finally, we investigate the \emph{full} identifiability of
  $\pxuw^{\bg,\bphi}$. To this end, we assume that $\ds=\dv=:d$ and the
  conditional distribution is given by
  \begin{align}
   &\log\psvw(\bs,\bv|\bw)
   =\sum_{i=1}^{d}q^{(i)}(s^{(i)},v^{(i)},\bw)
   -\log{Z}(\bw), \label{cond-wss}
  \end{align}
  where $Z(\bw)$ denotes the partition function. Unlike $\psvw$ in
  previous theorems, both $\bs$ and $\bv$ appear elementwisely on the
  right-hand side of~\eqref{cond-wss}.  With some fixed point
  $\bar{\bw}$ of $\bw$, we express
  ${\bq}^{(i)}(s,v,\bw):=q^{(i)}(s,v,\bw)-q^{(i)}(s,v,\bar{\bw})$. Then,
  a $d$-dimensional vector is defined by
  \begin{align}
   \bm{d}(\bs,\bv,\bw)&=\left(\frac{\partial^2}{\partial{s^{(1)}}\partial{v^{(1)}}}{\bq}^{(1)}(s^{(1)},v^{(1)},\bw),\dots,\frac{\partial^2}{\partial{s^{(d)}}\partial{v^{(d)}}}{\bq}^{(d)}(s^{(d)},v^{(d)},\bw)\right)^{\top}.
   \label{defi:d-vec}
  \end{align}  
  We now establish some conditions for full identifiability in the
  following theorem:
  \begin{Theo}
   \label{theo:full-IDF} Suppose that $\ds=\dv=:d$ and the following
   assumptions hold:
   \begin{enumerate}
    \item[(C1)] The conditional distribution $\psvw$ can be expressed
		as~\eqref{cond-wss}.
				
    \item[(C2)] There exist a single point $\bv'$ of $\bv$ and $d+1$
		points, $\bar{\bw}, \bw(1),\dots,\bw(d)$, such that the
		followings hold:
		\begin{itemize}
		 \item All of elements in
		       $\bm{d}(\bs,\bw):=\bm{d}(\bs,\bv',\bw)$ are
		       nonzeros for all $\bs$ and
		       $\bw\in\{\bw(1),\dots,\bw(d)\}$.

		 \item $\bm{d}(\bs,\bw(1)),\dots,\bm{d}(\bs,\bw(d))$ are
		       linearly independent for all $\bs$.
		\end{itemize}		
   \end{enumerate}
   Then, $\pxuw^{\bg,\bphi}$ is fully identifiable up to their
   permutation and elementwise functions.
  \end{Theo}
  The proof is given in Appendix~\ref{app:full-IDF}. Again, the
  dimensionality relation appeared in this theorem as well: For full
  identifiability, both latent vectors $\bs$ and $\bv$ have the same
  dimension, i.e., $\ds=\dv=:d$. In fact, this condition has been
  \emph{implicitly} used in EBM~\eqref{cond-xu}. Based on this
  dimensionality condition, the conditional distribution $\psvw$ is
  simplified as in Assumption~(C1). Assumption~(C2) would have similar
  implication as Assumption~(B4).
  \subsection{Removing the indeterminacy of nonlinear functions}
  As in existing works of nonlinear ICA, the indeterminacy of
  \emph{unknown} elementwise nonlinear functions in previous theorems
  makes it difficult to interpret an estimate of $\bg(\bx)$ in
  practice. Thus, it would be important to understand when this
  indeterminacy is removed. To this end, we establish the following
  proposition:
  \begin{Prop}
   \label{prop:nonlin-func} Suppose that all assumptions in
   Theorem~\ref{theo:full-IDF} hold. Furthermore, we assume that
   $\bq^{(i)}(s^{(i)},v^{(i)},\bw)$ for $i=1\dots,d$ fulfills the
   following equation:
   \begin{align}
    \frac{\partial^2}{\partial{s^{(i)}}\partial{v^{(i)}}}{\bq}^{(i)}(s^{(i)},v^{(i)},\bw)
    &=\alpha^{(i)}(v^{i},\bw),
    \label{model-assump}
   \end{align}
   where $\alpha^{(i)}(v^{(i)},\bw)$ is some function satisfying
   Assumptions~(C2). Then, $\pxuw^{\bg,\bphi}$ is partially identifiable
   with respect to $\bg$ up to a permutation of the elements in $\bg$,
   their scales and additive constants.
  \end{Prop}
  The proof of Proposition~\ref{prop:nonlin-func} is deferred to
  Appendix~\ref{app:nonlin-func}. Proposition~\ref{prop:nonlin-func}
  shows that $\bg(\bx)$ is essentially equal to the latent variable
  $\bs$ itself up to a permutation and scales.  Thus, there is no
  indeterminacy for elementwise nonlinear functions unlike previous
  theorems. This is a significant step for interpretation of an estimate
  of $\bg(\bx)$ in practice. Moreover,
  Proposition~\ref{prop:nonlin-func} would be surprising because the
  conclusion is the essentially same as \emph{linear}
  ICA~\citep{comon1994independent}, while our mixing function
  in~\eqref{generative-model} is \emph{nonlinear}.  A possible reason
  why Proposition~\ref{prop:nonlin-func} removes the indeterminacy of
  nonlinear functions in Theorem~\ref{theo:full-IDF} is that the
  nonlinear function ${\bq}^{(i)}(s^{(i)},v^{(i)},\bw)$ in $\psvw$ is
  presumably one of the main factors because~\eqref{model-assump}
  indicates that it is a linear function of $s^{(i)}$.
  %
  \subsection{Reverse data generative model}
  \label{ssec:reverse}
  The theoretical results in this section hold even when the data
  generative process is reversed: First, latent variables $\bs$ and
  $\bv$ follow a joint distribution $\psv(\bs,\bv)$. Then, given $\bs$
  and $\bv$, auxiliary data $\bw$ follows a conditional distribution
  $\pwsv(\bw|\bs,\bv)$. All of theorems and proposition above hold when
  $\pwsv(\bw|\bs,\bv)$ is equal to the right-hand side
  on~\eqref{cond-svw1} or~\eqref{cond-wss} up to the partition function,
  and all of the other conditions are satisfied. For instance, the same
  conclusion of Theorem~\ref{theo:partial-IDF} is holds when
  Assumption~(A1) is replaced with the following conditional
  distribution of $\bw$ given $\bs$ and $\bv$:
  \begin{align*}
   \pwsv(\bw|\bs,\bv)=\sum_{i=1}^{\ds}q^{(i)}(s^{(i)},\bv,\bw)
   -\log{Z}(\bs,\bv).
  \end{align*}
  More details including other theorems are discussed in
  Appendix~\ref{app:reverse}.
  
  This reverse view of data generative process yields interesting
  insights. First, we do not make \emph{any} assumptions on
  $\psv(\bs,\bv)$ as well as
  $\pxu^{\bg,\bphi}(\bx,\bu)=\psv(\bg(\bx),\bphi(\bu))\sqrt{\det(\bG_{\bg}(\bx))}\sqrt{\det(\bG_{\bphi}(\bu))}$.
  This means that unlike EBM~\eqref{cond-xu}, $\bx$ (or $\bs$) and $\bu$
  (or $\bv$) are not necessarily dependent, and can be even
  statistically independent. This property could be important for graph
  data where input data vectors associated with nodes can be
  i.i.d~\citep{pmlr-v80-okuno18a,okuno2019robust}. Another point is that
  the marginal distribution $\ps(\bs)$ has almost no assumptions. In
  contrast with linear ICA where the independence and nonGaussian of
  latent variables are assumed, in this work, $s^{(1)},\dots,s^{(\ds)}$
  can be Gaussians and/or dependent. In fact, we numerically demonstrate
  that $s^{(1)},\dots,s^{(\ds)}$ can be identified even when they are
  correlated Gaussians.
 \section{Related works}
 \label{sec:problem}
 Recently, there seem to exist two approaches for identifiability of
 latent variable models. One approach is to make structural assumptions
 on the mixing function $\bm{f}$ with the (nonconditional) independence
 assumption. Independent mechanism analysis (IMA) assumes that the
 columns in the Jacobian of $\bm{f}(\bs)$ are
 orthogonal~\citep{gresele2021independent}. More general function
 classes for identifiable models in nonlinear ICA are theoretically
 investigated in~\citet{buchholz2022function}.
 \citet{zheng2022identifiability} proved the identifiability of the
 latent variables with a structural sparse assumption on the Jacobian of
 $\bm{f}(\bs)$. Unlike this approach, we do not make structural
 assumptions on $\bm{f}$ and thus, our mixing function $\bm{f}$ is more
 general.
 
 Another approach is nonlinear ICA with auxiliary data, and makes the
 conditional independence assumption of latent variables given auxiliary
 data~\citep{hyvarinen2018nonlinear,sasaki2022representation}.  The
 conclusion is essentially the same as our work. However, our model
 includes two data generative models as in~\eqref{generative-model}, and
 $\bs$ and $\bv$ have equal or lower dimensions than $\bx$ and $\bu$
 (i.e., $\ds\leq\dx$ and $\dv\leq\du$) respectively, while previous work
 supposes that a single data general model, and $\bx$ and $\bs$ have the
 same dimension (i.e., $\dx=\ds$). Thus, we extend the data space to an
 embedded manifold.~\citet{NEURIPS2021_0cdbb4e6} also takes a similar
 manifold setting into account, but still considers a single data
 generative model.  Furthermore, we proved that the indeterminacies of
 nonlinear ICA are the same as linear one under certain conditions.
 
 Another relevant work is multiview nonlinear
 ICA~\citep{gresele2019incomplete,locatello2020weakly,lyu2022understanding}.
 Multiview nonlinear ICA usually considers multiple general models as
 in~\eqref{generative-model}, but assumes the shared latent variables in
 the generative models. On the other hand, we consider two latent
 variables $\bs$ and $\bv$, and our work could be more general because
 these variables d not necessarily have shared variables and can be even
 independent as discussed in Section~\ref{ssec:reverse}.
 \section{Application to graph embedding}
 This section applies Theorem~\ref{theo:full-IDF} to graph data, and
 establishes identifiable conditions, one of which includes an
 interesting insight. Finally, we propose a practical method for
 identifiable graph embedding.
  \subsection{Generative model for graph data}
  Graph data consists of data vectors associated with nodes and discrete
  link weights. The link weights usually reflect the strength of
  connections between nodes.  Here, we propose a generative model of
  data vectors and link weights by extending~\citet{pmlr-v80-okuno18a}
  and~\citet{okuno2019robust} to a latent variable model.  We first
  assume that a pair of the vectors of latent variables, $\bs\in\R{\ds}$
  and $\bs'\in\R{\ds}$, follow a joint probability distribution with the
  density $\pss(\bs,\bs')$. We do not necessarily assume that $\bs$ and
  $\bs'$ and/or the elements $s^{(1)},\dots,s^{(\ds)}$ in $\bs$ are
  independent. Then, data vectors $\bx$ and $\bx'$ are generated from
  the shared nonlinear mixing function $\bm{f}:\R{\ds}\to{\R{\dx}}$ with
  $\ds\leq\dx$ as
  \begin{align}
   \bx=\bm{f}(\bs)\quad\text{and}
   \quad\bx'=\bm{f}(\bs'). \label{generative-model-graph}
  \end{align}
  Furthermore, we assume that given $\bs$ and $\bs'$, the discrete link
  weight $w\in\{0,1,\dots,K\}$ is i.i.d. sampled from the conditional
  distribution $\pws(w|\bs,\bs')$. Here, we call $K$ as the
  \emph{maximum link state} throughout this paper. With the inverse
  $\bg=\bm{f}^{-1}$, the conditional distribution of $\bw$ given $\bx$
  and $\bx'$ can be expressed as
  $\pwx^{\bg}(w|\bx,\bx')=\pws(w|\bg(\bx),\bg(\bx'))$. Thus, we say
  $\pwx^{\bg}$ is identifiable when
  \begin{align}
   \pwx^{\bg}(w|\bx,\bx')=\pwx^{\tbg}(w|\bx,\bx')~~\Rightarrow~~\bg=\tbg.
   \label{identifiablity-graph}
  \end{align}

  This generative model for graph data is closed related to the proposed
  generative model in Section~\ref{sec:generative-model}: $\bs'$
  corresponds to $\bv$, and link weight $w$ is one dimensional discrete
  variable, which is a special case of auxiliary data $\bw$. Although
  this data generative process for graph data is reverse to the process
  in Section~\ref{sec:generative-model} yet theorems and proposition in
  Section~\ref{sec:identifiability} are applicable to this statistical
  model as discussed in Section~\ref{ssec:reverse}.
  \subsection{Identifiability conditions on graph data}
  As in Section~\ref{sec:identifiability}, we establish identifiable
  conditions for graph data.  Since it is essentially the same as
  Theorem~\ref{theo:full-IDF}, we summarize them as the following
  corollary without the proof:
  \begin{Coro}
   \label{coro:graphNICA} Suppose that the the following assumptions
   hold:
   \begin{enumerate}
    \item[(D1)] The conditional distribution of $w$ given $\bs$ and
		$\bs'$ can be expressed as follows:
		\begin{align}
		 \log\pws(w|\bs,\bs')&=\sum_{i=1}^{\ds}
		 q^{(i)}(w,s^{(i)},s^{(i)\prime})
		 -\log{Z}(\bs,\bs'). \label{cond-wss2}
		\end{align}
		where $Z(\bs,\bs')$ denotes the partition function.

    \item[(D2)] The maximum link state $K$ is equal to or larger than
		$\ds$, i.e., ${\ds}\leq{K}$.
	       
    \item[(D3)] There exists a single fixed point $\bar{\bs}$, and $d+1$
		points, $\bar{w}, w(1),\dots,w(\ds)$, of $w$ such that
		the followings are hold:
		\begin{itemize}
		 \item All of elements in
		       $\bm{d}(\bs,w):=\bm{d}(\bs,\bs'=\bar{\bs},w)$ are
		       nonzeros for all $\bs$ and
		       $w\in\{w(1),\dots,w(\ds)\}$ where
		       $\bm{d}(\bs,\bs',w)$ is similarly defined
		       as~\eqref{defi:d-vec}.
		       
		 \item $\bm{d}(\bs,w(1)),\dots,\bm{d}(\bs,w(\ds))$ are
		       linearly independent for all $\bs$.
		\end{itemize}
   \end{enumerate}
   Then, $\pwx^{\bg}$ is identifiable up to a permutation of the
   elements in $\bg$ and elementwise nonlinear functions.
  \end{Coro}  

  In contrast with Theorem~\ref{theo:full-IDF}, there is a remarkable
  point in Corollary~\ref{coro:graphNICA}: The maximum link state $K$
  has to be equal to or larger than $\ds$. This condition is necessary
  for the existence of $\ds$ linearly independent vectors
  $\bm{d}(w,\bs)$ in Assumption~(D3); If $\ds>K$, the linear independent
  $\ds$ vectors $\bm{d}(w(1),\bs),\dots,\bm{d}(w(\ds),\bs)$ never exist
  because of $w\in\{0,1,\dots,K\}$. Furthermore, this condition implies
  that the maximum link state $K$ may affect the performance for
  identifiability.
  In fact, Section~\ref{sec:exp} numerically supports this implication,
  and demonstrates that identifiability is clearly influenced by the
  maximum link state $K$.
  \subsection{Graph component analysis}
  \label{ssec:GCA}
  Here, we propose a practical method for estimating the inverse $\bg$,
  (Thus, the latent vector $\bs$ as well) from graph data. This method
  can be readily extended for a more general case in
  Theorem~\ref{theo:full-IDF}. The straightforward approach is to
  estimate the conditional distribution of $w$ given $\bx$ and $\bx'$.
  However, standard estimation methods such as maximum likelihood
  estimation require to compute the partition function
  $Z(\bx,\bx')$. Thus, when the maximum link state $K$ is very large
  (e.g., $K=\infty$), it is time consuming to compute it (e.g.,
  $Z(\bx,\bx')=\sum_{k=0}^{\infty}\pwxx(w=k|\bx,\bx')$).

  Alternatively, we take an approach of density ratio estimation which
  has been taken in existing methods for nonlinear
  ICA~\citep{sasaki2022representation} and graph
  embedding~\citep{satta2022graph}, and based on which general
  frameworks for estimation of unnormalized statistical models have been
  proposed~\citep{gutmann2011bregman,uehara2020unified}. In fact,
  Corollary~\ref{coro:graphNICA} holds even
  when~\eqref{full-identifiability} is replaced by the following
  distribution ratio:
  \begin{align}
   \frac{\pwx^{\bg}(w|\bx,\bx')}{\pw(w)}
   &=\frac{\pwx^{\tbg}(w|\bx,\bx')}{\pw(w)}
   ~~\Rightarrow~~\bg=\tbg.
  \end{align}
  Thus, one can estimate $\bg(\bx)$ through the distribution ratio
  estimation. This approach is close to contrastive
  learning~\citep{gutmann2012a} and appealing because it enables us to
  estimate $\bg$ to the infinite maximum link state (i.e., $K=\infty$)
  whose the partition function is difficult to compute in general. In
  addition, this approach can be extended for continuous $w$ without any
  efforts.
  
  Given a graph dataset,
  \begin{align*}
   \{w_{ij}\}_{1\leq{i}<j\leq{n}}~\text{and}~\{\bx_i\}_{i=1}^n,
  \end{align*}  
  where $w_{ij}(=w_{ji})$ is assumed to be symmetric, we estimate the
  ratio
  \begin{align}
   \log\frac{\pwx(w|\bx,\bx')}{\pw(w)}
   =\log\pwx(w|\bx,\bx')-\log\pw(w). \label{log-ratio}
  \end{align}
  To this end, we employ the following model for estimation of the
  ratio:
  \begin{align}
   r(w,\bx,\bx')&:=
   \sum_{i=1}^{\ds}[\psi^{(i)}(w,h^{(i)}(\bx),h^{(i)}(\bx'))]+b(w),
   \label{model-r}
  \end{align}
  where all of
  $\bh(\cdot)=(h^{(1)}(\cdot),\dots,h^{(\ds)}(\cdot))^{\top}$,
  $\psi^{(i)}(\cdot,\cdot,\cdot)$ and $b(\cdot)$ are models (e.g.,
  neural networks) and estimated from data. Compared
  with~\eqref{log-ratio}, the first summation on the right-hand side
  of~\eqref{model-r} models $\log\pwx(w|\bx,\bx')$ and the function form
  comes from Assumption~(D1) in Corollary~\ref{coro:graphNICA}. Thus,
  $\bh(\bx)$ corresponds to a model of the inverse $\bg$. The second
  term simply corresponds to $\log\pw(w)$.

  Based on the ratio approach, we adopt the Donsker-Varadhan variational
  estimation~\citep{ruderman2012tighter,belghazi2018mutual} because it
  has been experimentally demonstrated that it yields a sample-efficient
  method for nonlinear ICA~\citep{sasaki2022representation}. The
  objective function based on the Donsker-Varadhan variational
  estimation is defined as
  \begin{align}
   \Jdv(r)&:=-\sum_{w=0}^K\iint{r}(w,\bx,\bx')\pwxx(w,\bx,\bx')
   \intd\bx\intd\bx'\nonumber\\
   &+\log\left( \sum_{w=0}^K\iint e^{r(w,\bx,\bx')}
   \pw(w)\pxx(\bx,\bx')\intd\bx\intd\bx'\right). \label{dv-obj}
  \end{align}
  A simple calculation shows that $\Jdv(r)$ is minimized at the
  following distribution ratio:
  \begin{align*}
   r(w,\bx,\bx')=\log\frac{\pwxx(w,\bx,\bx')}{\pw(w)\pxx(\bx,\bx')}
   =\log\frac{\pwxx(w|\bx,\bx')}{\pw(w)}.
  \end{align*}
  
  In practice, we need to approximate $\Jdv(r)$ from the graph dataset,
  but may not be able to apply the standard low of large numbers based
  on the i.i.d. assumption because samples from $\pwxx(w,\bx,\bx')$ or
  $\pxx(\bx,\bx')$ can be non-i.i.d\footnote{Obviously, both
  $(\bx_i,\bx_j)$ and $(\bx_i,\bx_{j'})$ for $j\neq{j'}$ can be regarded
  as samples from $\pxx(\bx,\bx')$, but are clearly
  dependent.}. Fortunately, by applying the law of large numbers for
  doubly-indexed partially dependent random
  variables~\cite[Theorem~A.1]{okuno2019robust}, we can approximate
  $\Jdv(r)$ from the graph dataset as follows:
  \begin{align}
   \hJdv(r)&:=-\frac{1}{|\calI_n|}
   \sum_{(i,j)\in\calI_n}r(w_{ij},\bx_i,\bx_j)
   +\log\left(\frac{1}{|\calI_n|}\sum_{(i,j)\in\calI_n}
   e^{r(w^*_{ij},\bx_i,\bx_j)}\right), \label{emp-dv-obj}
  \end{align}
  where $w^*_{ij}$ is a random permutation of $w_{ij}$ with respect to
  $(i,j)$, $\calI_n:=\{(i,j)|1\leq{i}<{j}\leq{n}\}$ and $|\calI_n|$
  denotes the number of elements in $\calI_n$. Finally, we obtain an
  estimate of $\bg(\bx)$ as $\widehat{\bh}(\bx)$ by minimizing
  $\hJdv(r)$, and call this method as the \emph{graph component
  analysis} (GCA).
 \section{Illustration on artificial data}
 \label{sec:exp}
 This section numerically investigates how the latent dimension $\ds$
 and the maximum link state $K$ affect the performance of GCA as implied
 in Corollary~\ref{coro:graphNICA}. To this end, we generated i.i.d. $n$
 samples of latent vectors, $\{\bs_i\}_{i=1}^n$, from the following two
 distributions:
 \begin{itemize}
  \item \emph{Independent Laplace}: Latent vectors $\bs_i$ were sampled
	from the independent Laplace distribution as
	$\ps(\bm{s})\propto\prod_{i=1}^{\ds}\exp(-\sqrt{2}|s^{(i)}|)$.

  \item \emph{Correlated Gauss}: Latent vectors $\bs_i$ were sampled
	from a multivariate Gaussian distribution with mean $\bm{0}$ and
	covariance matrix $\bm{C}$ whose $(i,j)$-th elements $c^{(i,j)}$
	are given by
	\begin{align*}
	 c^{(i,j)}=\left\{
	 \begin{array}{ll}
	  1 & i=j\\
	  0.3 & j=i+1\\
	  0 & \text{otherwise}
	 \end{array}
	 \right.
	\end{align*}
	Thus, the elements in $\bs$ have clear correlation.
 \end{itemize}

 Given $\bs_i$ and $\bs_j~(i<j)$, we generated symmetric link weights
 $w_{ij}$ from the following the conditional distribution of $w$:
 \begin{align}  
  \pws(w=k|\bs,\bs')&=\frac{\exp(\sum_{i=1}^{\ds}
  \alpha^k_is^{(i)}s^{(i)\prime})}
  {\sum_{k'=1}^{K}\exp(\sum_{i=1}^{\ds}\alpha^{k'}_is^{(i)}s^{(i)\prime})},
  \label{exp-cond}
 \end{align}
 where with the independent uniform noise $\epsilon_{i}^k$ on $[0,1]$,
 $\alpha^k_i$ are randomly determined as
 \begin{align*}
  \alpha^k_i=\left\{
  \begin{array}{ll}
   1 + 0.1\epsilon_{i}^k& i={k}\\
   0.1\epsilon_{i}^k& i\neq{k}\\
  \end{array}
  \right.
 \end{align*}
 Based on the latent vectors $\bs_i$, data samples $\bx_i$ were
 generated as $\bx_i=\bm{f}(\bs_i)$ where $\bm{f}$ was modeled by a
 three-layer feedforward neural network: The number of units in each
 hidden layer was all the same as the data dimension $\dx$, and
 activation function was leaky ReLU with slope $0.2$.
 
 To estimate the inverse $\bg$, we modeled the feature vector $\bh(\bx)$
 by a five-layer neural network where the number of units in each hidden
 layer was $50$, and the activation function was ReLU, while the final
 layer had no activation function. Ratio model $r(w,\bx,\bx')$ was
 expressed as follows:
 \begin{align*}
  r(w,\bx,\bx')&=\sum_{i=1}^{\ds}
  \beta^{w}_ih^{(i)}(\bx)h^{(i)}(\bx')+b^{w},
 \end{align*}
 where $\beta^{w}_i$ and $b^w$ were parameters learned from graph data
 initialized as done for $\alpha^k_i$.  All parameters were optimized by
 Adam~\citep{kingma2015adam} with learning rate $10^{-4}$ where the
 minibatch size and total number of iterations were $100$ and $100,000$,
 respectively. For comparison, we learned the following energy-based
 model (EBM) appeared in~\citet{sasaki2018neural}
 and~\citet{khemakhem2020ice}:
 \begin{align*}
  \pxu^{\bh}(\bx|\bx')=\frac{1}{Z(\bh(\bx'))}
  \exp\left(\bh(\bx)^{\top}\bh(\bx')\right).
 \end{align*}
 Note that unlike GCA, EBM does not need link weight $w_{ij}$.
 $\bh(\bx)$ in EBM is modeled and learned in the same way as GCA.  More
 details are given in Appendix~\ref{app:EBM}.

 Since the conditional distribution~\eqref{exp-cond} in this experiment
 satisfies~\eqref{model-assump} in Proposition~\ref{prop:nonlin-func},
 the inverse function $\bg(\bx)$ should be linear to $\bs$. Thus, as
 done in~\citet{sasaki2022representation}, we measured the performance
 by the mean absolute correlation between the test latent vectors
 $\bs^{\mathrm{te}}$ and estimated feature vectors
 $\widehat{\bh}(\bx^{\mathrm{te}})$ on test data $\bx^{\mathrm{te}}$
 where the number of test samples were $10,000$. Thus, a larger
 correlation value means a better performance.

 We first present how the performance changes to the latent dimension
 $\ds$ in Fig.\ref{fig:meancorr_dim} when the data dimension and maximum
 link state are fixed at $\dx=6$ and $K=10$, respectively. GCA (red
 lines) shows very high correlation values when $\ds$ is less than or
 equal to $\dx=6$. However, as $\ds$ is more than $\dx=6$, the
 performance of GCA is quickly decreased. This would be because the
 mixing function $\bm{f}$ is not bijective in the case of $\ds>\dx$, and
 cannot be an embedding as assumed in
 Section~\ref{sec:generative-model}. On the other hand, EBM (blue lines)
 shows low correlations on the range of $\ds$. EBM assumes that pairs of
 $\bx_i$ and $\bx_j$ are statistically dependent, but
 $\bx_1,\dots,\bx_n$ are independent in this experiment because
 $\bs_1,\dots,\bs_n$ are i.i.d. This means that EBM might not be
 suitable for graph data.
 
 Next, the averages of the mean absolute correlations are presented in
 Fig.\ref{fig:meancorr} when the maximum link state $K$ is modified with
 $\dx=\ds=6$. Note that EBM does not use link weights, the maximum link
 weight is fixed at $K=10$ only for learning EBM.  When the maximum link
 state $K$ is less than the latent dimension $\ds=6$, the mean
 correlations of GCA (red lines) are clearly small. On the other hand,
 the performance is significantly improved as $K$ is equal to or gets
 larger than $\ds=6$. This result is fairly consistent, and supports the
 implication of Corollary~\ref{coro:graphNICA}. Again, since
 $\bx_1,\dots,\bx_n$ are i.i.d., EBM (blue lines) does not estimate the
 latent vector $\bs$.

 \begin{figure}[t]
  \centering \subfigure[Independent
  Laplace]{\includegraphics[width=0.45\textwidth]{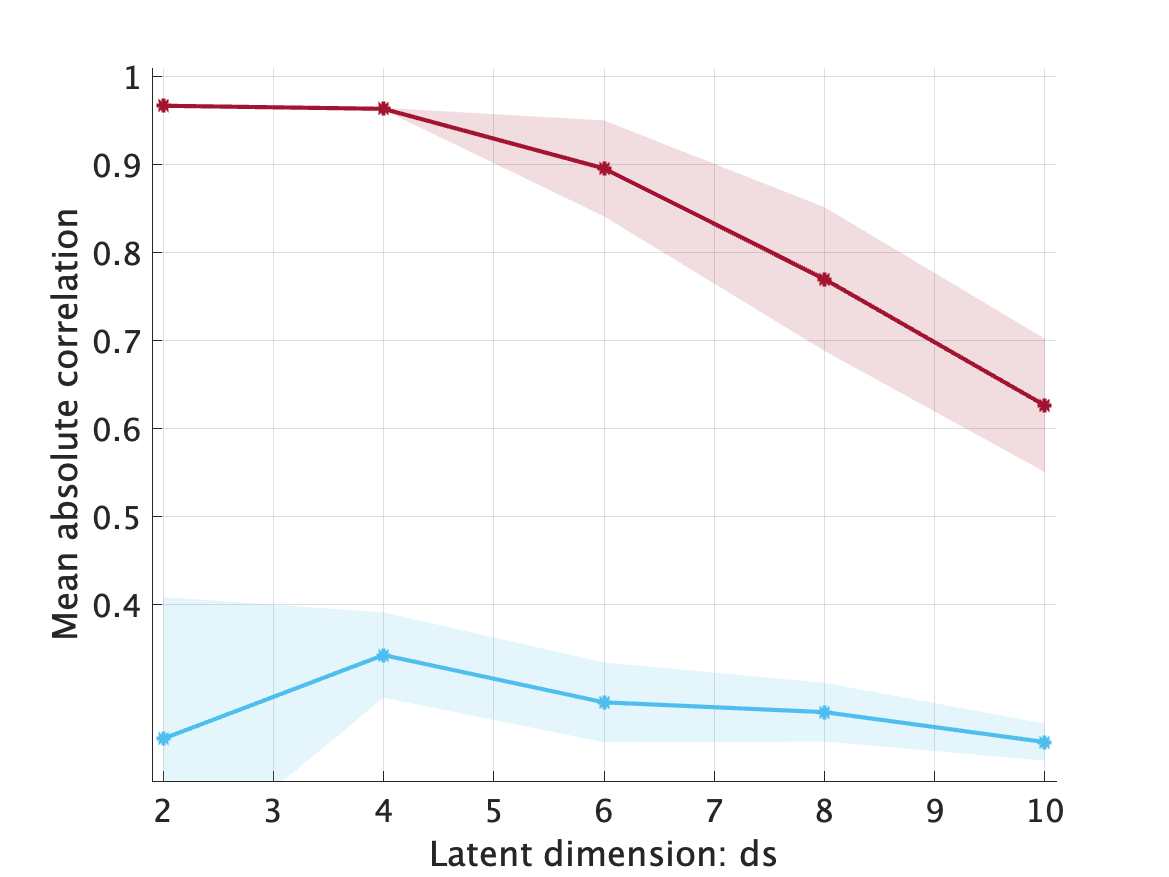}}
  \subfigure[Correlated
  Gauss]{\includegraphics[width=0.45\textwidth]{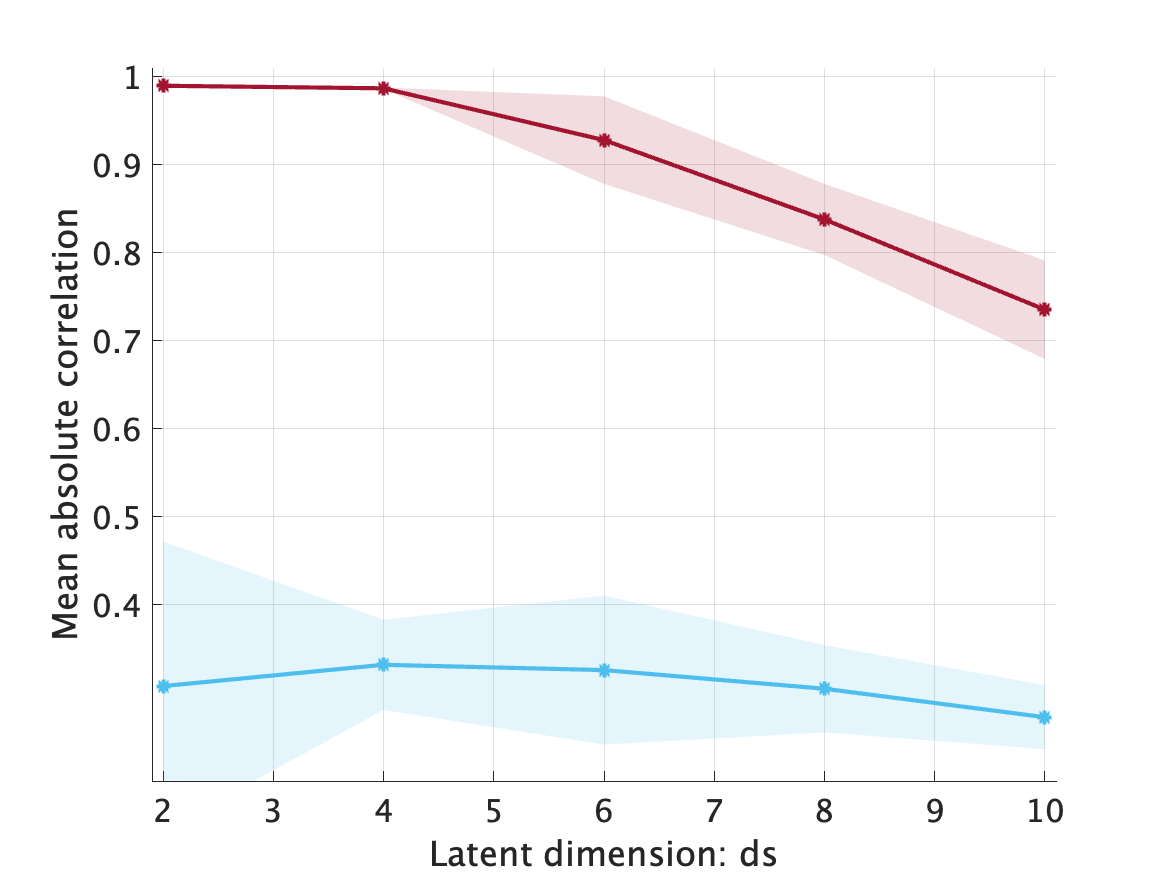}}
  \caption{\label{fig:meancorr_dim} Mean absolute correlation over the
  latent dimension $\ds$ when $(\dx, K, n)=(6, 10, 10000)$. The red and
  blue lines are for GCA and EBM, respectively. Each marker represents
  the average of the mean absolute correlations over $10$ runs, and the
  shaded region is the standard deviation.}
 \end{figure} 

 \begin{figure}[t]
  \centering \subfigure[Independent
  Laplace]{\includegraphics[width=0.45\textwidth]{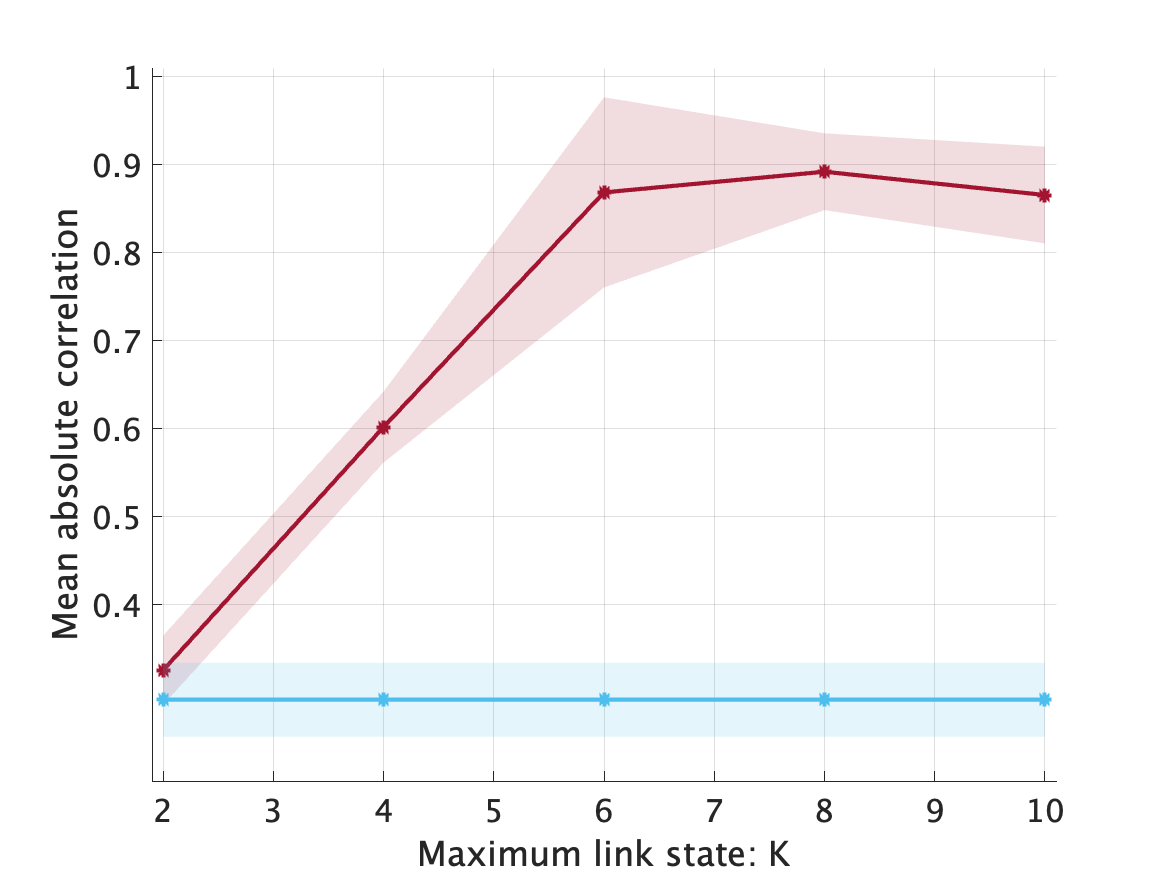}}
  \subfigure[Correlated
  Gauss]{\includegraphics[width=0.45\textwidth]{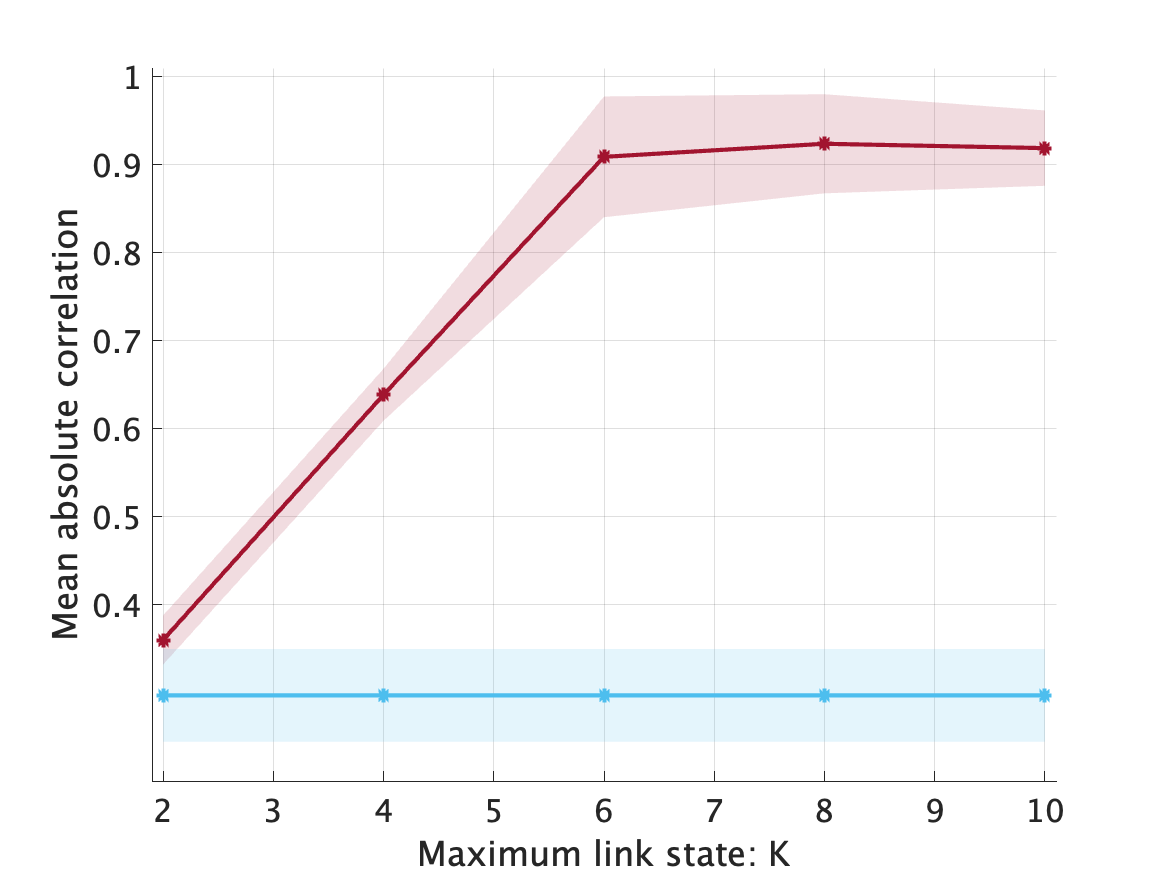}}
  \caption{\label{fig:meancorr} Mean absolute correlation over the
  number of link states when $\ds=\dx$ and $(\dx, n)=(6, 10000)$. The
  red and blue lines are for GCA and EBM, respectively.  For EBM, the
  blue line is flat because link weights are not required and the
  maximum link state is fixed at $K=10$ in learning EBM.}
 \end{figure} 

 \section{Conclusion}
 This paper proposed a statistical model of two latent vectors with
 single auxiliary data. Based on the proposed model, we established
 conditions for partial and full identifiability of the proposed model.
 Interestingly, the dimensionality relation among latent vectors and
 auxiliary data would be important conditions for identifiability. When
 interpreting our model from the reverse data generative process, our
 work has weak assumptions on the joint and marginal distribution of two
 latent vectors. In addition, we proved that the indeterminacy of our
 model has the same indeterminacy as linear ICA under certain
 conditions. The proposed statistical model and identifiability theory
 were applied to graph data, and then one of the identifiability
 conditions includes an interesting implication: The maximum link weight
 is an important factor for identifiability. Based on the application to
 graph data, we proposed a practical method for identifiable graph
 embedding called graph component analysis (GCA). Numerical experiments
 demonstrated that the performance of GCA clearly depends on the maximum
 link weight as implied in our theoretical results.

 \appendix 

 \section{Useful lemmas}
 Our proof is based on the following lemma proved
 in~\citet{sasaki2022representation}:
 \begin{Lemm}[Lemma~10 in~\citet{sasaki2022representation}]
  \label{lem:diagonals} Suppose that $\bm{D}(v)$ is an $n$ by $n$
  diagonal matrix whose diagonals $d_{i}(v), i=1\dots,n$ is a function
  of $v\in\R{}$, and $\bm{A}$ and $\bm{B}$ are $n$ by $n$ constant
  matrices. Furthermore, the following assumptions are made:
  \begin{enumerate}
   \item[(1)] There exist $n$ points $v_1, v_2,\dots,v_n$ such that
	      $\bm{d}(v_1),\bm{d}(v_2),\dots,\bm{d}(v_n)$ are linearly
	      independent where
	      $\bm{d}(v):=(d^{(1)}(v),d^{(2)}(v),\dots,d^{(n)}(v))^{\top}$
	      is the vector of the diagonal elements in $\bm{D}(v)$.
	       
   \item[(2)] $\bm{A}$ and $\bm{B}$ are of full-rank.
  \end{enumerate}
  Then, when $\bm{A}\bm{D}(v)\bm{B}$ is a diagonal matrix at least at
  $n$ points $v_1, v_2,\dots, v_n$, then both $\bm{A}$ and $\bm{B}$ are
  diagonal matrices multiplied by a permutation matrix.
 \end{Lemm}

 We also rely on the following lemma:
 \begin{Lemm}
  \label{lem:jacobian} If $\bm{f}:\R{\ds}\to\R{\dx}$ with $\ds\leq{\dx}$
  is a smooth embedding, the Jacobian of its inverse $\bg:=\bm{f}^{-1}$
  has rank $\ds$.
 \end{Lemm}
 \begin{proof}
  By definition, $\bs=\bg\circ\bm{f}(\bs)$. By computing its gradient
  with respect to $\bs$,
  \begin{align*}
   \bm{I}_{\ds}=\bJ_{\bg}(\bm{f}(\bs))^{\top}\bJ_{\bm{f}}(\bs).
  \end{align*}
  Since $\bm{f}$ is a smooth embedding and $\bJ_{\bm{f}}(\bs)$ has rank
  $\ds$, the rank of $\bJ_{\bg}$ has to be $\ds$.
 \end{proof}
 \section{Proof of Theorem~\ref{theo:partial-IDF}}
 \label{app:proof-partial-IDF}
 \begin{proof}
  We first recall that the conditional distribution of $\bx$ and $\bu$
  given $\bw$ under $\bg$ and $\bphi$ is given by
  \begin{align}
   \log\pxuw^{\bg,\bphi}(\bx,\bu|\bw)
   &=\sum_{i=1}^{\ds}q^{(i)}(g^{(i)}(\bx),\bphi(\bu),\bw)-\log{Z}(\bw)
   \nonumber\\&\qquad+\frac{1}{2}\log\det(\bG_{\bg}(\bx))
   +\frac{1}{2}\log\det(\bG_{\bphi}(\bu)).  \label{cond-inverty}
  \end{align}
  With a single point $\bar{\bw}$ of $\bw$ in Assumption~(A3), we have
  \begin{align*}
   \log\pxuw^{\bg,\bphi}(\bx,\bu|\bw)
   -\log\pxuw^{\bg,\bphi}(\bx,\bu|\bar{\bw})
   =\sum_{i=1}^{\ds}{\bq}^{(i)}(g^{(i)}(\bx),\bphi(\bu),\bw)
   -\log\frac{Z(\bw)}{Z(\bar{\bw})},
  \end{align*}
  where
  ${\bq}^{(i)}(s,\bv,\bw):=q^{(i)}(s,\bv,\bw)-q^{(i)}(s,\bv,\bar{\bw})$.
  By assuming that the same conditional distribution
  as~\eqref{cond-inverty} exists under $\tbg$ and $\tbphi$, we have
  $\pxuw^{\bg,\bphi}(\bx,\bu|\bw)=\pxuw^{\tbg,\tbphi}(\bx,\bu|\bw)$
  implying
  \begin{align}
   \sum_{i=1}^{\ds}{\bq}^{(i)}(g^{(i)}(\bx),\bphi(\bu),\bw)
   =\sum_{i=1}^{\ds}{\bq}^{(i)}(\tg^{(i)}(\bx),\tbphi(\bu),\bw).
   \label{ide-cond-tmp}
  \end{align}
  
  Let us express $\tbs:=\tbg(\bx)$ and $\tbv:=\tbphi(\bu)$, and recall
  $\bs=\bg(\bx)=\bg\circ\tilde{\bm{f}}(\tbs)$ and
  $\bv=\bphi(\bu)=\bphi\circ\tbvphi(\tbv)$ where
  $\tilde{\bm{f}}:=\tbg^{-1}$ and $\tbvphi:=\tbphi^{-1}$.  Then, we
  express~\eqref{ide-cond-tmp} as
  \begin{align}
   \sum_{i=1}^{\ds}{\bq}^{(i)}(s^{(i)},\bv,\bw)
   =\sum_{i=1}^{\ds}{\bq}^{(i)}(\ts^{(i)},\tbv,\bw),
   \label{ide-cond}
  \end{align}
  and then compute the cross-derivative of both sides
  on~\eqref{ide-cond} with respect to $\ts^{(l)}$ and $\ts^{(m)}$ for
  $l\neq{m}$ as follows:
  \begin{align}
   \sum_{i=1}^d\left\{\frac{\partial^2s^{(i)}}
   {\partial{\ts^{(l)}}\partial{\ts^{(m)}}}\bq^{(i,1)}(s^{(i)},\bv,\bw)
   +\parder{s^{(i)}}{\ts^{(l)}}\parder{s^{(i)}}{\ts^{(m)}}
   \bq^{(i,2)}(s^{(i)},\bv,\bw)\right\}=0, \label{ide-cond-cross-der}
  \end{align}
  where $\bq^{(i,1)}(s,\bv,\bw):=
  \frac{\partial}{\partial{s}}\bq^{(i)}(s,\bv,\bw)$ and
  $\bq^{(i,2)}(s,\bv,\bw):=
  \frac{\partial^2}{\partial{s^2}}\bq^{(i)}(s,\bv,\bw)$.  We further
  compactly express~\eqref{ide-cond-cross-der} as the following vector
  form:
  \begin{align}
   \bm{a}_{l,m}(\tbs)^{\top}\bbq^{\prime}(\bs,\bv,\bw)=0,
   \label{cross-der-vector}
  \end{align}  
  where
  \begin{align*}
   \hspace{-3mm}\bm{a}_{l,m}(\tbs)&:=\left(
   \frac{\partial^2s^{(1)}}{\partial{\ts^{(l)}}\partial{\ts^{(m)}}},\dots,
   \frac{\partial^2s^{(\ds)}}{\partial{\ts^{(l)}}\partial{\ts^{(m)}}},
   \parder{s^{(1)}}{\ts^{(l)}}\parder{s^{(1)}}{\ts^{(m)}},\dots,
   \parder{s^{(\ds)}}{\ts^{(l)}}\parder{s^{(\ds)}}{\ts^{(m)}}
   \right)^{\top}\hspace{-1mm}\in\R{2\ds}\\ 
   \hspace{-3mm}\bbq^{\prime}(\bs,\bv,\bw)&:=
   \left(\bq^{(1,1)}(s^{(1)},\bv,\bw),\dots,
   \bq^{(\ds,1)}(s^{(\ds)},\bv,\bw),\right.\\
   &\left.\qquad\qquad\qquad\qquad
   \bq^{(1,2)}(s^{(1)},\bv,\bw),
   \dots,\bq^{(\ds,2)}(s^{(\ds)},\bv,\bw) \right)^{\top}\in\R{2\ds}.
  \end{align*}
  Next, we define the following two vectors where the $i$-th elements in $\bv$
  and $\bw$ are fixed:
  \begin{align*}
   \bv^{(i)}&=(v^{(1)},\dots,v^{(i)}=\bar{v}^{(i)}\dots,v^{(\dv)})^\top\\
   \bw^{(i)}&=(w^{(1)},\dots,w^{(i)}=\bar{w}^{(i)}\dots,w^{(\dw)})^\top
  \end{align*}
  These two vectors enable us to express~\eqref{cross-der-vector} as
  \begin{align*}
   \bm{a}_{l,m}(\tbs)^{\top}\bar{\bm{Q}}^{\prime}(\bs)=\bm{0},
  \end{align*}
  where 
  \begin{align*}
   \bar{\bm{Q}}^{\prime}(\bs)&:=
   \left[\bbq^{\prime}(\bs,\bv^{(1)},\bw),\dots,
   \bbq^{\prime}(\bs,\bv^{(\dv)},\bw),
   \right.\\&\left.\qquad\qquad
   \bbq^{\prime}(\bs,\bv,\bw^{(1)}),\dots,
   \bbq^{\prime}(\bs,\bv,\bw^{(\dw)})\right]\biggr|_{\bv=\bv',\bw=\bw'}
   \in\R{2\ds\times{(\dv+\dw)}}.
  \end{align*}
  By Assumptions~(A2-3), the right inverse of
  $\bar{\bm{Q}}^{\prime}(\bs)$ exists and thus
  $\bm{a}_{l,m}(\tbs)=\bm{0}$, i.e.,
  \begin{align}
   \parder{s^{(i)}}{\ts^{(l)}}\parder{s^{(i)}}{\ts^{(m)}}=0,
   \label{parder-zero}
  \end{align}  
  for all $i$ and $(l,m)$ for $l\neq{m}$. 

  Finally, we complete the proof by showing the Jacobian of $\bs$ with
  respect to $\tbs$ is of full rank. With
  $\bs=\bg\circ\tilde{\bm{f}}(\tbs)$,
  \begin{align*}
   \bJ_{\bs}(\tbs)=\bJ_{\bg}(\tilde{\bm{f}}(\tbs))^{\top}
   \bJ_{\tilde{\bm{f}}}(\tbs).
  \end{align*}
  Since $\tilde{\bm{f}}$ is a smooth embedding, $\bJ_{\tilde{\bm{f}}}$
  has rank $\ds$, and Lemma~\ref{lem:jacobian} ensures that $\bJ_{\bg}$
  also has rank $\ds$. Thus, $\bJ_{\bs}(\tbs)$ has $\ds$ rank and is of
  full rank. Thus, by~\eqref{parder-zero}, each row of $\bJ_{\bs}(\tbs)$
  has a single nonzero element. This means that each $s_i$ is a function
  of a single variable in $\tbs$, which completes the proof.
 \end{proof}

 \section{Proof of Theorem~\ref{theo:partial-IDF-mild}}
 \label{app:partial-IDF-mild}
 \begin{proof}
  We begin by differentiating~\eqref{ide-cond} with respect to $s^{(l)}$
  and $v^{m}$ as
  \begin{align}
   &\bqw^{(l)}(\bw)\frac{\partial^2}{\partial s^{(l)}\partial{v}^{(m)}}
   \qsv^{(l)}(s^{(l)},\bv)\nonumber\\
   &\quad=\sum_{i=1}^{\ds}\sum_{j=1}^{\dv}
   \bqw^{(i)}(\bw)\parder{\ts^{(i)}}{s^{(l)}}\parder{\tv^{(j)}}{v^{(m)}}
   \frac{\partial^2}{\partial\ts^{(i)}\partial{\tv}^{(j)}}
   \qsv^{(i)}(\ts^{(i)},\tbv), \label{parder-sl-vm}
  \end{align}
  where Assumption~(B1) was applied. Eq.\eqref{parder-sl-vm} can be
  expressed as the following matrix form:
  \begin{align}
   \bm{D}(\bw)\bm{Q}_{\mathrm{sv}}(\bs,\bv)
   =\bJ_{\tbs}(\bs)^{\top}\bm{D}(\bw)\bm{Q}_{\mathrm{sv}}(\tbs,\tbv)
   \bJ_{\tbv}(\bv), \label{parder-sl-vm-matrix}
  \end{align}
  where $\bm{D}(\bw):=\diag(\bqw^{(\ds)}(\bw),\dots,\bqw^{(\ds)}(\bw))$
  and $\bm{Q}_{\mathrm{sv}}(\bs,\bv)$ is a $\ds$ by $\dv$ matrix with
  the $(l,m)$-th element
  $\frac{\partial^2}{\partial{s}^{(l)}\partial{v}^{(m)}}\qsv^{(l)}(s^{(l)},\bv)$.
  By Assumptions~(B2-3), $\bm{Q}_{\mathrm{sv}}(\bs,\bv=\bv')$ has the
  right inverse. Thus, by denoting the right inverse by
  $\bm{Q}_{\mathrm{sv}}^{\dagger}(\bs)$, we have
  from~\eqref{parder-sl-vm-matrix},
  \begin{align}
   \bm{D}(\bw)
   =\bJ_{\tbs}(\bs)^{\top}\bm{D}(\bw)\underbrace{\bm{Q}_{\mathrm{sv}}(\tbs,\tbv')
   \bJ_{\tbv}(\bv')\bm{Q}_{\mathrm{sv}}^{\dagger}(\bs)}_{\bm{M}(\bs)},
   \label{parder-sl-vm-matrix2}
  \end{align}
  where $\tbv':=\tbphi\circ\bvphi(\bv')$. Assumption~(B4) ensures that
  $\bm{D}(\bw)$ is full rank for $\bw\in\{\bw(1),\dots,\bw(\ds)\}$. In
  addition, $\bJ_{\tbs}(\bs)$ has also full rank because $\bm{f}$ is a
  smooth embedding. Thus, $\bm{M}(\bs)\in\R{\ds\times{\ds}}$ must be of
  full rank. By Assumption~(B4), the vectors of the diagonals in
  $\bm{D}(\bw)$ for $\bw\in\{\bw(1),\dots,\bw(\ds)\}$ are linearly
  independent. Therefore, applying Lemma~\ref{lem:diagonals}
  to~\eqref{parder-sl-vm-matrix2} indicates that $\bJ_{\tbs}(\bs)$ is
  the product of a diagonal and permutation matrix. This completes the
  proof.
 \end{proof}
  
 \section{Proof of Theorem~\ref{theo:full-IDF}}
 \label{app:full-IDF}
 \begin{proof}
  We first compute the partial derivative of~\eqref{ide-cond} with
  respect to $s^{(l)}$ and $v^{(m)}$ under Assumption~(C1) as
  \begin{align}
   \frac{\partial^2}{\partial{s}^{(l)}\partial{v^{(m)}}}
   \bq^{(l)}(s^{(l)},v^{(m)},\bw)
   =\sum_{i=1}^d\parder{\ts^{(i)}}{s^{(l)}}
   \parder{\tv^{(i)}}{v^{(m)}}
   \frac{\partial^2}{\partial{\ts^{(i)}}\partial{\tv^{(i)}}}
   \bq^{(i)}(\ts^{(i)},\tv^{(i)},\bw).  \label{cross-der}
  \end{align}
  We further compactly express~\eqref{cross-der} in a matrix form as
  \begin{align}
   \bD(\bs,\bv,\bw)
   =(\bm{J}_{\tbs}(\bs))^{\top}\bD(\bw,\tbs,\tbv)
   \bm{J}_{\tbv}(\bv),
   \label{matrix-univ3}
  \end{align}
  where $\bD(\bs,\bv,\bw)$ is a diagonal matrix with the $i$-th diagonal
  $\frac{\partial^2}{\partial{s}^{(i)}\partial{v^{(i)}}}\bq^{(i)}(s^{(i)},v^{(i)},\bw)$.
  
  As in previous proofs, our goal is to show that $\bm{J}_{\tbs}(\bs)$
  is the product of a permutation and diagonal matrix. To this end, we
  substitute $\bv'$ into $\bv$ in~\eqref{matrix-univ3} and have
  \begin{align}
   \bD_{\mathrm{v'}}(\bs,\bw)
   &=(\bm{J}_{\tbs}(\bs))^{\top}\bD_{\mathrm{\tv'}}(\tbs,\bw)\bm{J}'_{\tbv},
  \label{two-fixed-tmp}
  \end{align}
  where $\bD_{\mathrm{v'}}(\bs,\bw):=\bD(\bs,\bv',\bw)$,
  $\bD_{\mathrm{\tv'}}(\bs,\bw):=\bD(\bs,\tbv=\tbphi\circ\bvphi(\bv'),\bw)$
  and $\bm{J}'_{\tbv}:=\bm{J}_{\tbv}(\bv')$. Since $\bvphi$ is a smooth
  embedding, $\bm{J}'_{\tbv}$ is of full rank. Furthermore,
  Assumption~(C2) ensures that $\bD_{\mathrm{v'}}(\bs,\bw)$ is of
  full-rank for $\bw\in\{\bw(1),\dots,\bw(\ds)\}$ and all $\bs$.  Thus,
  for each $\bs$ as well as $\tbs=\tbg\circ{\bm{f}}(\bs)$, applying
  Lemma~\ref{lem:diagonals} to the right-hand side
  on~\eqref{two-fixed-tmp} under Assumption~(C2) indicates that both
  $\bm{J}_{\tbs}(\bs)$ and $\bm{J}'_{\tbv}$ are the product of diagonal
  and permutation matrix. The proof is completed.
 \end{proof}
 \section{Proof of Proposition~\ref{prop:nonlin-func}}
 \label{app:nonlin-func}
 \begin{proof}
  Here, some notations are inherited from
  Appendix~\ref{app:full-IDF}. From~\eqref{model-assump}, we have
  \begin{align*}
   \frac{\partial^2}{\partial{s^{(i)}}
   \partial{v^{(i)}}}\bq^{(i)}(\bw,s_i,v_i)
   &=\alpha^{(i)}(\bw,v^{(i)})\\
   \frac{\partial^2}{\partial{\ts^{(i)}}\partial{\tv^{(i)}}}\bq^{(i)}(\bw,\ts_i,\tv^{(i)})
   &=\alpha^{(i)}(\bw,\tv^{(i)}).
  \end{align*}
  Thus, once $\bw$ is fixed at $\bw=\bw'$, both
  $\bD_{\mathrm{v'}}(\bs,\bw=\bw')$ and
  $\bD_{\mathrm{\tv'}}(\tbs,\bw=\bw')$ in~\eqref{matrix-univ3} are
  constant diagonal matrices whose diagonals are given by
  $(\alpha^{(1)}(\bw,v^{(1)}),\dots,\alpha^{(d)}(\bw,v^{(d)}))|_{\bw=\bw',\bv=\bv'}$
  and
  $(\alpha^{(1)}(\bw,\tv^{(1)}),\dots,\alpha^{(d)}(\bw,\tv^{(d)}))|_{\bw=\bw',\tbv=\tbphi\circ\bvphi(\bv')}$,
  respectively. Thus, we re-express these matrices simply by
  $\bD_{\mathrm{w'v'}}$ and $\bD_{\mathrm{w'\tv'}}$. Then, it follows
  from~\eqref{two-fixed-tmp} that $\bm{J}_{\tbs}(\bs)$ can be expressed
  as
  \begin{align}
   \bm{J}_{\tbs}(\bs)
   =(\bD_{\mathrm{w'\tv'}}\bm{J}'_{\tbv})^{-\top}\bD_{\mathrm{w'v'}}.
   \label{jacobian-v}
  \end{align}
  As proved in Appendix~\ref{app:full-IDF}, $\bm{J}'_{\tbv}$ is the
  product of a diagonal and permutation matrix. Thus,~\eqref{jacobian-v}
  indicates that there exist a nonzero constant $\gamma^{(i)}$ and
  permutation index $\pi(i)\in\{1,\dots,\ds\}$ such that
  \begin{align*}
   \parder{\ts^{(i)}}{s^{(\pi(j))}}
   =\left\{
   \begin{array}{cc}
   \gamma^{(i)} & i=j \\
    0 & i\neq{j}
   \end{array}  
   \right.
  \end{align*}
  This completes the proof.
 \end{proof}
 \section{Reverse generative model}
 \label{app:reverse}
 Here, we give details of why identifiability is guaranteed in the
 reverse generative model as well. The proofs for all of theorems come
 with the partial derivatives of~\eqref{ide-cond} or its variants. Thus,
 it suffices to show~\eqref{ide-cond} holds in the reverse generative
 model. As an alternative to $\psvw(\bs,\bv|\bw)$ in
 Theorem~\ref{theo:partial-IDF}, we assume that the conditional
 distribution of $\bw$ given $\bs$ and $\bv$ is given in the reverse
 generative model by
 \begin{align*}
  \log\pwsv(\bw|\bs,\bv)=\sum_{i=1}^{\ds}q^{(i)}(s^{(i)},\bv,\bw)
  -\log{Z}(\bs,\bv).
 \end{align*}
 Then, by $\bs=\bg(\bx)$ and $\bv=\bphi(\bu)$, we obtain the conditional
 distribution of $\bw$ given $\bx$ and $\bu$ as 
 \begin{align*}
  \log\pwxu^{\bg,\bphi}(\bw|\bx,\bu)
  =\sum_{i=1}^{\ds}q^{(i)}(g^{(i)}(\bx),\bphi(\bu),\bw)
  -\log{Z}(\bx,\bu).
 \end{align*}
 With a single point $\bar{\bw}$ of $\bw$ in Assumption~(A3), we have
 \begin{align*}
  \log\pwxu^{\bg,\bphi}(\bw|\bx,\bu)-\log\pxuw^{\bg,\bphi}(\bar{\bw}|\bx,\bu)
  =\sum_{i=1}^{\ds}{\bq}^{(i)}(g^{(i)}(\bx),\bphi(\bu),\bw).
 \end{align*}
 Finally,
 $\pxuw^{\bg,\bphi}(\bx,\bu|\bw)=\pxuw^{\tbg,\tbphi}(\bx,\bu|\bw)$
 enables us to obtain~\eqref{ide-cond}. By exactly following the steps
 after~\eqref{ide-cond}, we reach the same conclusion as Theorem~1 if
 the other assumptions are retained.  It can be shown in the reverse
 generative model that the conclusions of
 Theorems~\ref{theo:partial-IDF-mild} and~\ref{theo:full-IDF} hold in
 the same way above.
 \section{Learning energy-based models}
 \label{app:EBM}
 This appendix gives details for learning energy-based models in
 Section~\ref{sec:exp}. Here, we take a similar approach as GCA and
 estimate the following density ratio:
 \begin{align*}
  \frac{\pxcx(\bx|\bx')}{\px(\bx)}.
 \end{align*}
 To estimate it, the objective function based on the Donsker-Varadhan
 variational estimation is used as
 \begin{align}
  \Jdv(r)&:=-\iint{r}(\bx,\bx')\pxx(\bx,\bx')
  \intd\bx\intd\bx'\nonumber\\
  &\qquad\qquad+\log\left(\iint e^{r(\bx,\bx')}
  \px(\bx)\px(\bx')\intd\bx\intd\bx'\right). 
 \end{align}
 A simple calculation shows that $\Jdv(r)$ is minimized at the following
 distribution ratio:
 \begin{align*}
  r(\bx,\bx')=\log\pxcx(\bx|\bx')-\log\px(\bx)
 \end{align*}
 Thus, $r(\bx,\bx')$ is modeled as 
 \begin{align*}
  r(\bx,\bx')=\bh(\bx)^{\top}\bh(\bx')+a(\bh(\bx)),
 \end{align*}
 where $a(\cdot)$ is modeled as a single layer neural network without
 activation function. Then, we empirically approximate $\Jdv(r)$ from
 data samples as follows:
 \begin{align}
  \hJdv(r)&:=-\frac{1}{|\calI_n|}
  \sum_{(i,j)\in\calI_n}r(\bx_i,\bx_j)+\log\left(\frac{1}{|\calI_n|}\sum_{(i,j)\in\calI_n}
  e^{r(\bx_i,\bx^*_j)}\right), \label{emp-dv-obj}
 \end{align}
  where $\bx^*_{j}$ is a random permutation of $\bx^*_{j}$ with respect
  to $j$. Learning EBM is completed by minimizing $\hJdv(r)$.

\bibliographystyle{apalike}      
\bibliography{../../../papers}

\end{document}